\documentclass{svproc}
\usepackage{times}
\usepackage{tabularx}
\usepackage{multicol}
\usepackage[bookmarks,bookmarksnumbered,colorlinks, urlcolor={black}, linkcolor={black},citecolor={black}]{hyperref}
\usepackage{mathtools}
\usepackage{xspace}
\usepackage{amsmath,amssymb,mathrsfs}
\usepackage{thmtools}
\usepackage{thm-restate}
\usepackage{graphicx} 
\usepackage{todonotes} 
\usepackage{tikz}
\usepackage{proof}
\usepackage{etoolbox}
\usepackage{textcomp}
\usepackage{comment}
\usepackage{enumerate}   
\usepackage[ruled,vlined,noend]{algorithm2e}
\usepackage{color}
\usepackage{subcaption}
\usepackage{bm}
\usepackage{sidecap}

\SetCommentSty{mycommfont}
\usepackage{url}

\providecommand{\fref}[1]{Fig.~\ref{#1}} 

\providecommand{\Language}[1]{\ensuremath{\mathcal{L}(#1)}}

\providecommand{\gobble}[1]{}

\providecommand{\term}{\ensuremath{\mathrm{term}}}
\providecommand{\goal}{\ensuremath{\mathrm{goal}}}

\newrobustcmd*{\mysquare}[1]{\tikz{\filldraw[draw=black,fill=#1] (0,0) rectangle
(0.2cm,0.2cm);}}

\newrobustcmd*{\mycircle}[1]{\tikz{\filldraw[draw=black,fill=#1] (0,0) circle
[radius=0.1cm];}}

\newrobustcmd*{\mytriangle}[1]{\tikz{\filldraw[draw=black,fill=#1] (0,0) --
(0.2cm,0) -- (0.1cm,0.2cm);}}

\definecolor{amethyst}{rgb}{0.6, 0.4, 0.8}
\definecolor{gre}{rgb}{0.4, 1.0, 0.5} 
\definecolor{gra}{rgb}{0.8, 0.8, 0.85} \definecolor{re}{rgb}{1.0, 0.5, 0.5}

\definecolor{setformula}{rgb}{0.0, 0.0, 0.0}
\definecolor{equivrel}{rgb}{0.0, 0.0, 0.0}

\let\emptyset\varnothing

\definecolor{darkblue}{rgb}{0.15,0.09,0.3}

\definecolor{formula}{rgb}{0.08,0.35,0.08}
\newcommand\form[1]{\textcolor{formula}{\bm{#1}}}
\definecolor{varcolour}{rgb}{0.0,0.0,0.0}
\newcommand\var[1]{\textcolor{varcolour}{\bm{#1}}}

\newrobustcmd*{\true}{{\sl True}\xspace}
\newrobustcmd*{\false}{{\sl False}\xspace}

\DeclareFontFamily{OT1}{pzc}{}
\DeclareFontShape{OT1}{pzc}{m}{it}{<-> s * [1.10] pzcmi7t}{}
\DeclareMathAlphabet{\mathpzc}{OT1}{pzc}{m}{it}

\newcounter{tecounter}
\newenvironment{tightenumerate}
{
    \begin{list}{\arabic{tecounter}\addtocounter{tecounter}{1})}{%
    \setcounter{tecounter}{1}
        \setlength{\leftmargin}{08pt}
        \setlength{\topsep}{1pt}
        \setlength{\partopsep}{0pt}
        \setlength{\itemsep}{2pt}
        \setlength\labelwidth{0pt}}
        \ignorespaces}
{\unskip\end{list}}

\SetAlFnt{\small}
\SetAlCapFnt{\normalsize}
\SetAlCapNameFnt{\normalsize}
\usepackage[noend]{algorithmic}
\algsetup{linenosize=\tiny}

\providecommand{\fine}[1]{\ensuremath{\widetilde{#1}}}
\providecommand{\inv}[1]{\ensuremath{{#1}^{-1}}}

\providecommand{\of}{\ensuremath{\circ}}
\providecommand{\trto}[4]{\ensuremath{\operatorname{\textsc{TransTo}}({#2}\xrightarrow{#3}{#4})}^{#1}}
\providecommand{\sde}[1]{\ensuremath{\operatorname{\textsc{Sde}}({#1})}}
\providecommand{\pgprod}{\ensuremath{\otimes}}

\newcommand{\defeq}{\ensuremath{\coloneqq}}
\providecommand{\True}{\ensuremath{\operatorname{{True}}}\xspace}
\providecommand{\False}{\ensuremath{\operatorname{False}}\xspace}

\providecommand{\reachedv}[2]{\textcolor{setformula}{\ensuremath{\mathcal{V}^{#1}_{#2}}}}
\providecommand{\reachings}[2]{\textcolor{setformula}{\ensuremath{\mathcal{S}^{#1}_{#2}}}}

\providecommand{\exactreachings}[2]{\textcolor{setformula}{\ensuremath{\mathbb{S}}^{#1}_{#2}}}

\providecommand{\compatablew}[2]{\textcolor{setformula}{\ensuremath{\mathcal{W}}^{#1}_{#2}}}

\providecommand{\nameseek}[1]{{\sc Seek${}_{#1}$}}
\providecommand{\nameseekp}{\nameseek{\var{x}}}

\providecommand{\nameseekplm}{\nameseek{\var{x}, \var{\lambda}}}

\providecommand{\veq}[3]{\textcolor{equivrel}{#1\underset{#2}{\sim}#3}}
\providecommand{\nveq}[3]{\textcolor{equivrel}{#1\underset{#2}{\not\sim}#3}}

\providecommand{\aNd}{{\sc and}\xspace}
\providecommand{\Or}{{\sc or}\xspace}

\newcommand\blfootnote[1]{%
  \begingroup
  \renewcommand\thefootnote{}\footnote{#1}%
  \addtocounter{footnote}{-1}%
  \endgroup
}

\begin{document}

\title{
Finding plans subject to stipulations on what information they divulge%
}

\author{Yulin Zhang\inst{1}, Dylan A. Shell\inst{1} \and Jason M. O'Kane\inst{2}}
\institute{
Texas A\&M University, College Station TX, USA\\
\and 
University of South Carolina, Columbia SC, USA%
}

\maketitle

\begin{abstract}
Motivated by applications where privacy is important, we study planning
problems for robots acting in the presence of an observer.  We first formulate
and then solve planning problems subject to stipulations on the information
divulged during plan execution---the appropriate solution concept being both a
plan and an information disclosure policy.  We pose this class of problem under
a worst-case model within the framework of procrustean graphs, formulating the
disclosure policy as a particular type of map on edge labels.  We devise
algorithms that, given a planning problem supplemented with an information
stipulation, can find a plan, associated disclosure policy, or both jointly, if
and only if some exists. The pair together, 
comprising the plan and associated disclosure policy, 
may depend subtly on
additional information available to the
observer, such as whether the observer knows the robot's plan (e.g., leaked via
a side-channel).  
Our implementation finds a plan and a suitable disclosure policy,
jointly, when any such pair exists, albeit for small problem instances.
\end{abstract}

\blfootnote{This work was supported by NSF awards 
\href{http://nsf.gov/awardsearch/showAward?AWD_ID=1453652}{IIS-1453652},
\href{http://nsf.gov/awardsearch/showAward?AWD_ID=1527436}{IIS-1527436},
and 
\href{http://nsf.gov/awardsearch/showAward?AWD_ID=1526862}{IIS-1526862}.}

\vspace*{-26pt}
\section{Introduction}
\vspace*{-6pt}
In 2017, iRobot announced that they intended to sell maps of people's homes,
as generated by their robot vacuum cleaners.  The result was a public
outcry~\cite{privacynews}.  It is increasingly clear that, as robots become part
of our everyday lives, the information they could collect (indeed, may
\emph{need} to collect to function) can be both sensitive and valuable.
Information about a robot's internal state and its estimates of the world's
state are leaked by status displays, logged data, actions executed, and
information outputted --- often what the robot is tasked with doing. The tension
between utility and privacy is fundamental.

Typically, robots strive to decrease uncertainty.
Some prior work, albeit limited, has illustrated how to cultivate uncertainty,
examining how to constrain a robot's beliefs so that it never learns sensitive
information (cf.~\cite{OKa08,ddf2015,zhang18complete}). In so doing, one
precludes sensitive information being disclosed to any adversary. But not
disclosing secrets by simply never knowing any, limits the
applicability of the approach severely. This paper proposes a more general,
wider-reaching model for privacy, beyond mere ing\'{e}nue robots.

This article posits a potentially adversarial observer and then stipulates
properties of what shall be divulged.  The stipulation describes information
that must be communicated (being required to perform the task) as well as
information (confidential information potentially violating the user's privacy)
that shouldn't be.  Practical scenarios where this model applies include: ($i$)
privacy-aware care robots that
assist the housebound,
providing nursing care; ($ii$) inspection of sensitive facilities by robots to
certify compliance with regulatory agreements, whilst protecting other
proprietary or secret details; ($iii$) sending data remotely to computing
services on untrusted cloud infrastructure.  


\setlength{\textfloatsep}{5pt}
\begin{SCfigure}
\hspace*{-5pt}
\begin{minipage}{0.25\textwidth}
\vspace*{-13pt}
\includegraphics[scale=0.44]{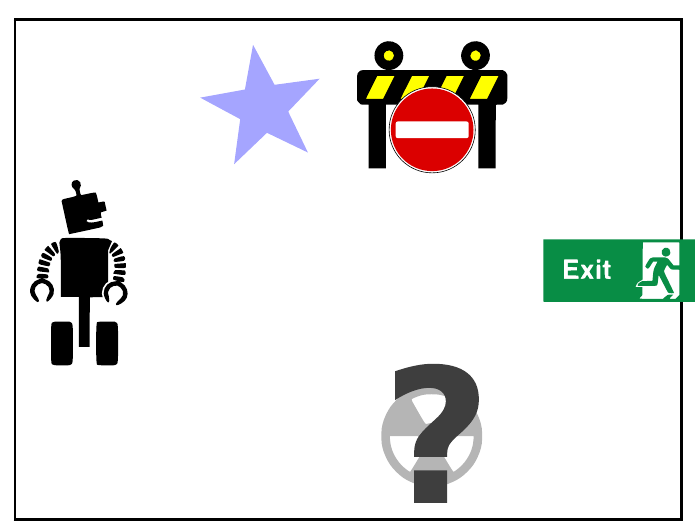}\\
\small\centering Pebble bed facility
\end{minipage}%
\begin{minipage}{0.28\textwidth}
\vspace*{-15pt}
\includegraphics[scale=0.44]{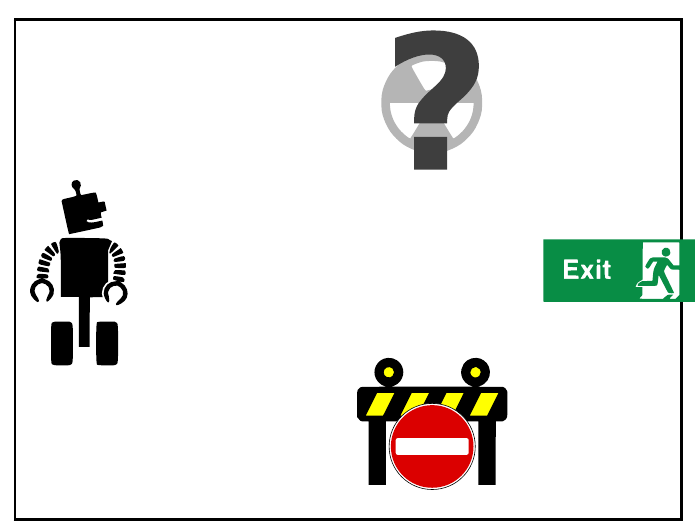}\\
\small\centering Breeder reactor
\end{minipage}%
\caption{{\bf Nuclear Site Inspection\;\;}%
A robot inspects a nuclear facility by taking a measurement at the 
`{\bf ?}' location, which depends on the facility type.
But the type of the facility is sensitive information that it must not be divulged to any external observers.
\label{fig:nuclear}}
\end{SCfigure}

Figure~\ref{fig:nuclear} illustrates a scenario which, though simplistic, is
rich enough to depict several aspects of the problem.  The task requires that a
robot determine whether some facility's processing of raw radioactive material
meets international treaty requirements or not.  The measurement procedure
itself depends on the type of facility as the differing physical arrangements
of  `pebble bed' and `breeder' reactors necessitate different actions.  First,
the robot must actively determine the facility type (checking for the presence
of the telltale blue light in the correct spot).  Then it can go to a location
to make the measurement, the measurement location corresponding with the
facility type.  But the facility type is sensitive information and the robot
must ascertain the radioactivity state while ensuring that the facility type is not
disclosed.

What makes this scenario interesting is that the task is rendered
infeasible immediately if one prescribes a policy to ensure that the robot never gains
sensitive information.  Over and above the (classical) question of how to balance
information-gathering and progress-making actions, the robot must control what
it divulges, strategically increasing uncertainty as needed, precisely limiting
and reasoning about the `knowledge gap' between the external observer and
itself.  To solve such problems, the robot needs a carefully constructed plan
and must establish a policy characterizing what information it divulges, the
former achieving the goals set for the robot, the latter respecting all
stipulated constraints---and, of course, each depending on the other.

\subsection{Contributions and itinerary}

This paper contributes the first formulation, to our knowledge, of planning
where solutions can be constrained so as to require that some information be
communicated and other information obscured subject to an adversarial model of
an observer. Nor do we know of other work where both a plan and some notion of
an interface (the disclosure policy, in our terminology) can both be solved for
jointly. The paper is organized as follows: after discussion of related work,
Section~\ref{sec:preliminary} develops the preliminaries, notation, and
formalism, Section~\ref{section:divulgedplan} addresses an important technical
detail regarding an observer's background knowledge, and
Section~\ref{section:search} finding plans that satisfy the stipulations.   
The last section reports experiments conducted with our implementation. 

\section{Related work}
\label{sec:related}

An important topic in HRI is expressive action (e.g.,
see~\cite{takayama11expressing}).  In recent years there has been a great deal
of interest in mathematical models that enable generation of communicative
plans.  Important formulations include those of
\cite{dragan17robot,knepper17implicit}, proposing plausible models for human
observers (from the perspectives of presumed cost efficiency, surprisal, or
generalizations thereof). In this prior work, conveying information becomes
part of an optimization objective, whereas we treat it as a constraint instead.
Both \cite{dragan17robot} and \cite{knepper17implicit} are probabilistic in nature, here
we consider a worst-case model that is arguably more suitable for privacy considerations:
We ask what an observer can plausibly infer via the history of its received
observations.  In doing so, we are influenced by the philosophy of
LaValle~\cite{lavalle2006planning}, following his use of the term information
state (I-state) to refer to a representation of information derived from a
history of observations.\gobble{\footnotemark}  Finally, since parts of our stipulations may require
concealing information, we point out there is also recent work in deception (see
\cite{masters17deceptive,dragan15deception}) and also obfuscation
\cite{Wu2016Obfuscator}.






\section{The model: worlds, robots and observers}
\label{sec:preliminary}

Figure~\ref{fig:modeloverview} illustrates the three-way relationships
underlying the setting we examine.  Most fundamentally, a robot executes a
\emph{plan} to achieve some goal in the \emph{world}, and the coupling of these
two elements generates a stream of observations and actions. Both the plan and
the action--observation stream are disclosed, though potentially only partially,
to a third party, we term the \emph{observer}. The observer uses the stream, its
knowledge of the plan, and also other known structure to infer properties about
the interaction. Additionally, a stipulation is provided specifying particular
properties that can be learned by the observer.  We formalize these elements in
terms of p-graphs and label maps (see~\cite{saberifar18pgraph}).

\setlength{\textfloatsep}{10pt}
\begin{SCfigure}[60]
\centering
\includegraphics[scale=0.7]{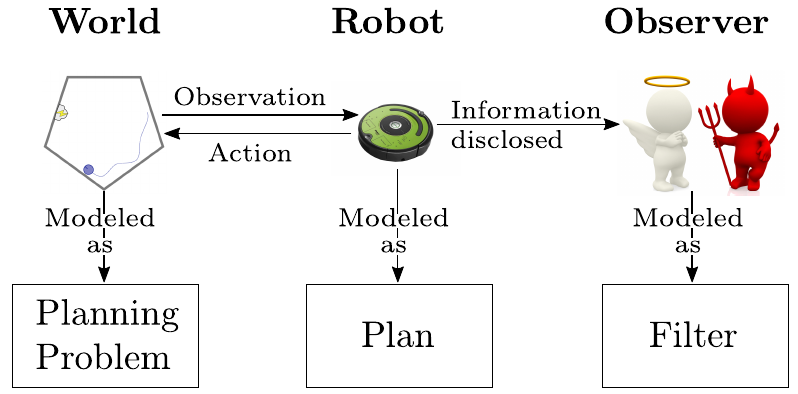}
\caption{An overview of the setting: the robot is modeled abstractly as realizing
a plan to achieve some goal in the world and a third party observes, modeled as a
filter.  All three, the world, plan, and filter have concrete
representations as p-graphs.%
\label{fig:modeloverview}}
\end{SCfigure}

\vspace*{-20pt}
\subsection{P-graph and its interaction language}
We will start with the definition of p-graphs~\cite{saberifar18pgraph} and related properties:
\begin{definition}[p-graph]
A \emph{p-graph} is an edge-labelled directed bipartite graph with $G=(V_y\cup
V_u, Y, U, V_0)$, where 
\begin{tightenumerate}
\item the finite vertex set 
$V(G)\defeq V_y\cup V_u$, whose elements are also called
\emph{states}, comprises two disjoint subsets: the
\emph{observation vertices} $V_y$ and the \emph{action vertices} $V_u$,

\item each edge $e$ originating at an observation vertex bears a set of
observations ${Y(e) \subseteq Y}$, containing \emph{observation labels}, and
leads to an action vertex, \item each edge $e$ originating at an action vertex
bears a set of actions ${U(e) \subseteq U}$, containing \emph{action labels},
and leads to an observation vertex, and
\item a non-empty set of states $V_0$ are designated as \emph{initial states},
which may be either exclusively action states ($V_0\subseteq V_u$) or
exclusively observation states ($V_0\subseteq V_y$).
\end{tightenumerate}
\end{definition}

An \emph{event} is an action or an observation. Respectively, they make up the
sets $U$ and $Y$, which are called the p-graph's \emph{action space} and
\emph{observation space}. We will also\gobble{ have occasion to} write $Y(G)$ and $U(G)$
for the observation space and action space of $G$. Though that is a slight
abuse of notation, the initial states will be written $V_0(G)$, similarly. 

Intuitively, a p-graph abstractly represents a (potentially non-deterministic)
transition system where transitions are either of type `action' or
`observation' and these two alternate. The following definitions make this
idea precise.
%
\vspace*{-3pt}
\begin{definition}[transitions to]
For a given p-graph $G$ and two states $v,w\in V(G)$, a sequence of events
$\ell_1, \dots, \ell_k$ \emph{transitions in $G$ from $v$ to $w$} if there exists a
sequence of states $v_1, \dots, v_{k+1}$, such that $v_1=v$, $v_{k+1}=w$, and
for each $i=1, \dots, k$, there exists an edge $v_i\xrightarrow{E_i}{v_{i+1}}$
for which $\ell_i\in E_i$, and $E_i$ is a subset of $Y(G)$ if $v_i$ is in
$V_y$, or a subset of $U(G)$ if $v_i$ is in $V_u$.
\end{definition}

Concisely, we let the predicate $\trto{G}{v}{s}{w}$ hold if there is some way of
tracing $s$ on $G$ from $v$ to $w$, i.e., it is \True iff $v$ transitions to $w$
under execution $s$.  Note, when $G$ has non-deterministic transitions,
$v$ may transition to multiple vertices under the same execution.  We only
require that $w$ be one of them.
\vspace*{-3pt}
\begin{definition}[executions and interaction language]
An \emph{execution} on a p-graph $G$ is a finite sequence of events $s$, if
there exists some $v\in V_0(G)$ and some $w\in V(G)$ for which
$\trto{G}{v}{s}{w}$.  The set of all executions on $G$ is called the
\emph{interaction language} (or, briefly, just \emph{language}) of $G$ and is
written $\Language{G}$.
\end{definition}

Given any edge $e$, if $U(e)=L_e$ or $Y(e)=L_e$, we speak of $e$ \emph{bearing}
the set $L_e$. 
%
\begin{definition}[joint-execution]
A \emph{joint-execution} on two p-graphs $G_1$ and $G_2$ is a sequence of events $s$
that is an execution of both $G_1$ and $G_2$, written as $s\in
\Language{G_1}\cap \Language{G_2}$. The p-graph producing all the
joint-executions of $G_1$ and $G_2$ is their tensor product graph with initial
states $V_0(G_1) \times V_0(G_2)$, which we denote $G_1\pgprod G_1$. 
\end{definition}

A vertex from $G_1\pgprod G_2$ is as a pair $(v_1, v_2)$, where $v_1\in V(G_1)$
and $v_2\in V(G_2)$. Next, the relationship between the executions and vertices
is established.
\begin{definition}
The set of vertices reached by execution $s$ in $G$, denoted $\reachedv{G}{s}$, 
are the vertices to which the execution $s\in \Language{G}$ transitions,
starting at an initial state. Symbolically, $\reachedv{G}{s} \defeq\{v \in
V(G)\;|\;\exists v_0 \in V_0(G), \trto{G}{v_0}{s}{v}\}.$ Further, the set of
executions reaching vertex $v$ in $G$ is written as
$\reachings{G}{v}\defeq\{s\in \Language{G}\,|\,v\in \reachedv{G}{s}\}$.
\end{definition}


The naming here serving to remind that $\mathcal{V}$ describes sets of vertices,
$\mathcal{S}$ describes sets of strings/executions.  
The collection of sets $\{\reachings{G}{v_0}, \reachings{G}{v_1}, \dots,
\reachings{G}{v_i}\dots\}$ can be used to form an equivalence relation
$\veq{}G{}$ over executions, under which $\veq{s_1}{G}{s_2}$ if and only if
\mbox{$\reachedv{G}{s_1}=\reachedv{G}{s_2}$}.  This equivalence relation partitions the
executions in $\Language{G}$ into a set of non-empty equivalence classes:
$\Language{G}/{\veq{}{G}{}}=\{[r_0]_G, [r_1]_G, [r_2]_G, \dots\}$, where each
equivalence class is $[r_i]_G=\{s\in \Language{G}\,|\,\veq{r_i}{G}{s}\}$ and $r_i$
is a representative execution in $[r_i]_{G}$.  The intuition is that any two
executions that transition to identical sets of vertices are, in an important
sense, indistinguishable.

We shall consider systems where the vertices of a
p-graph constitute the state that is stored, acted upon, and/or
represented---they are, thus, akin to a `sufficient statistic'.

\begin{definition}[state-determined]
A p-graph $G$ is in a \emph{state-determined} presentation, or is in
\emph{state-determined} form, if $\,\forall s\in \Language{G},
|\reachedv{G}{s}|=1$.  
\end{definition}

The procedure to expand any p-graph $G$ into a state-determined presentation
$\sde{G}$ can be found in Algorithm 2 of~\cite{saberifar18pgraph}. 
The language of p-graphs is not affected by state-determined
expansion, i.e., $\Language{G}=\Language{\sde{G}}$.


Next, one may start with vertices and ask about the executions reaching those
vertices. (Later, this will be part of how an observer makes inferences about
the world.)



\begin{definition}
\label{vertextoequiv}
Given any set of vertices $B\subseteq V(G)$ in p-graph $G$, the set of
executions that reach exactly (i.e. reach and reach only) $B$ is
$\exactreachings{G}{B}\defeq(\cap_{v\in B} \reachings{G}{v})\setminus \cup_{v\in
(V(G)\setminus B)} \reachings{G}{v}$.
\end{definition}

Above, the $\cap_{v\in B} \reachings{G}{v}$ represents the set of executions
that reach every vertex in $B$. By subtracting the ones that also reach the
vertices outside $B$, $\exactreachings{G}{B}$ describes the set of executions
that reach exactly $B$. In Figure~\ref{fig:stringbehindvertices}, the executions
reaching $w_3$ are represented as $\reachings{G}{w_3}=\{a_1o_1, a_2o_1\}$. But
the executions reaching and reaching only $\{w_3\}$ are
$\exactreachings{G}{\{w_3\}}=\{a_1o_1\}$ since $a_2o_1$ also reaches $w_4$.
Specifically, the equivalence class $[r_i]_G$ contains the executions that reach
exactly $\reachedv{G}{r_i}$, so we have
$[r_i]_G=\exactreachings{G}{\reachedv{G}{r_i}}$.

\vspace*{-15pt}
\begin{SCfigure}[40]
{\includegraphics[scale=0.85]{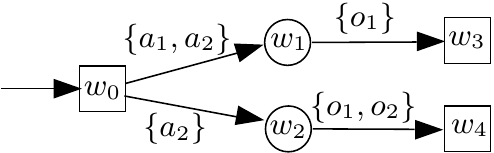}}\hspace{20pt}
\caption{An example showing the difference between `reaches' and `reaches
exactly' as distinguished in notation as $\reachings{G}{w}$ and
$\exactreachings{G}{\{w\}}$.\label{fig:stringbehindvertices}}
\vspace*{2pt}
\end{SCfigure}
\vspace*{-20pt}

\subsection{Planning problems and plans}

In the p-graph formalism, planning problems and plans are defined as
follows~\cite{saberifar18pgraph}.

\begin{definition}[planning problems and plans]
A \emph{planning problem} is a p-graph $W$ along with a {goal region} $V_{\goal}
\subseteq V(W)$; a \emph{plan} is a p-graph $P$ equipped with a 
{termination region} $V_{\term} \subseteq V(P)$.
\end{definition}

Planning problem $(W, V_{\goal})$ is solved by some plan $(P, V_{\term})$ if the
plan always terminates (i.e., reaches $V_{\term}$) and only terminates at a
goal. Said with more precision:

\begin{definition}[solves]
\label{def:solves}
A plan $(P,V_{\term})$ \emph{solves} a planning problem $(W,V_{\goal})$ if there
is some integer which bounds length of all joint-executions, and for each
joint-execution and any pair of nodes $(v \in V(P),w \in V(W))$ reached by that
execution simultaneously, the following conditions hold:
\begin{tightenumerate}
\item if $v$ and $w$ are both action nodes and, for every label borne by each
edge originating at $v$, there exist edges originating at $w$ bearing the same
action label; \item if $v$ and $w$ are both observation nodes and, for every
label borne by each edge originating at $w$, there exist edges originating at
$v$ bearing the same observation~label; 
\item if $v \in V_{\term}$ and then $w \in V_{\goal}$;
\item if $v \notin V_{\term}$ then some extended joint-execution exists,
continuing from $v$ and $w$, that does reach the termination region.
\end{tightenumerate}
\end{definition}

In the above, properties 1) and 2) describe a notion of safety; property 3) of
correctness; and 4) of liveness.  In the previous definition, there is an upper
bound on joint-execution length.  We say that plan $(P,V_{\term})$ is
\emph{$c$-bounded} if, $\forall s\in \Language{P}$, $|s|\leq c$. 

\subsection{Information disclosure policy, divulged plan, and observer}

The agent who is the observer sees a stream of the robot's actions and observations, and uses
them to build estimates (or to compute general properties) of the robot's
interaction with the world. But the observer's access to this information will
usually be imperfect---either by design, as a consequence of real-world
imperfections, or some combination of both.  Conceptually, this is a form of
partial observability in which the stream of symbols emitted as part of the
robot's execution is distorted into to the symbols seen by the observer (see
Figure~\ref{fig:dataflow}).  For example, if some pairs of actions are
indistinguishable from the perspective of the observer, this may be expressed
with a function that maps those pairs of actions to the same value.  In this
paper, this barrier is what we have been referring to (informally, thus far)
with the phrase \emph{information disclosure policy}. It is formalized as a
mapping from the events in the robot's true execution in the world p-graph to
the events received by the observer.  

\vspace*{-6pt}
\begin{SCfigure}
 	\centering
	\includegraphics[scale=0.73]{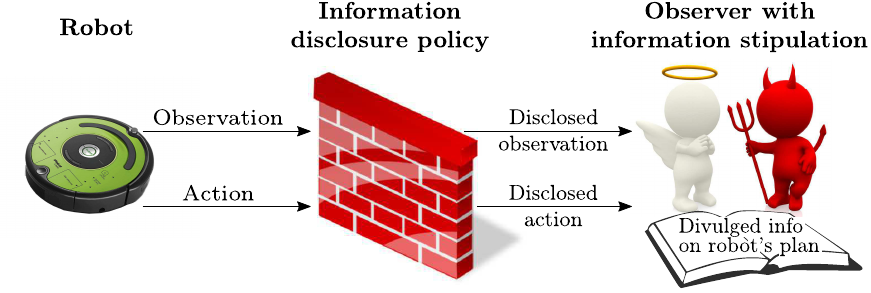}
\caption{The information disclosure policy, divulged plan and information
stipulation.  Even when the observer is a strong adversary, the disclosure
policy and divulged plan can limit the observer's capabilities effectively.}
\label{fig:dataflow}
\vspace*{7pt}
\end{SCfigure}
\vspace*{-18pt}

\vspace*{-2pt}
\begin{definition}[Information disclosure policy]
An \emph{information disclosure policy} is a label map $h$ on p-graph
$G$, mapping from elements in the combined observation and action space
$Y(G)\cup U(G)$ to some set of events $X$.
\end{definition}
\vspace*{-2pt}

The word `policy' hints at two interpretations: first, as something given as
a predetermined arrangement (that is, as a rule); secondly, as something to be
sought (together with a plan).  Both senses apply in the present work;
the exact transformation describing the disclosure of information will be used
first (in Section~\ref{section:seekplan}) as a specification and then, later (in
Section~\ref{section:seekplanlabelmap}) as something which planning algorithms
can produce. How the information disclosure policy is realized in some setting
depends on which sense is apt: it can be interpreted as describing observers
(showing that for those observers unable to tell $y_i$ from $y_j$, the
stipulations can be met), or it can inform robot operation (the stipulations
require that the robot obfuscate $u_\ell$ and $u_m$ via means such as explicit
concealment, sleight-of-hand, misdirection, etc.)

The observer, in addition, may also have imperfect knowledge of robot's plan,
which is leaked or communicated from the side-channel. The \emph{disclosed plan}
is also modeled as a p-graph, which may be weaker than knowing the actual plan.
A variety of different types of divulged plan are introduced later (in
Section~\ref{section:divulgedplan}) to model different prior knowledge available
to an observer; as we will show, despite their differences, they can be treated
in a single unified way.

The next step is to provide formal definitions for the ideas just described.  In
the following, we refer to $h$ as the map from the set $Y\cup U$ to some set
$X$, and refer to its preimage $\inv{h}$ as the map from $X$ to subsets of
$Y\cup U$. The notation for a label map $h$ and its preimage $\inv{h}$ is
extended in the usual way to sequences and sets: we consider sets of events,
executions (being sequences), and sets of executions. They are also extended to
p-graphs in the obvious way, by applying the function to all edges.

For brevity's sake, the outputs of~$h$ will be referred to simply as `the
image space.' The function $h$ may either preserve information (when a
bijection) or lose information (with multiple inputs mapped to one 
output). 
The loss of information is felt in $Y\cup U$ by the extent to which some element of
$Y \cup U$ grows under $\inv{h} \of h$, and for all  $\ell\in Y\cup U$,
$\inv{h}\of h(\ell)\supseteq \{\ell\}$.  In contrast, starting from $x \in X$,
the uncertainty, apparent via set cardinality under $\inv{h}$, is washed out
again when pushed forward to the image space $X$ via $h\of\inv{h}$,  i.e.,
$\forall x\in X$, $h\of\inv{h}(x)=\{x\}$.

\vspace*{-2pt}
\begin{definition}[I-state graph]
For planning problem $(W,V_{\goal})$, plan $(P,V_{\term})$ and information
disclosure policy $h: Y(W)\cup U(W) \to X$, an observer's \emph{I-state graph}
$I$ is a p-graph, whose inputs are from the image space of  $h$ (i.e., $Y(I)\cup
U(I)=X$), with $\Language{I} \supseteq h[\Language{W}]$.  
The action space and observation space of $I$ are also written as $X_u=U(I)$ and
$X_y=Y(I)$. 
\end{definition}

Inherited from the property of $h\of\inv{h}$, for any I-state graph $I$, we have
$I=h\of\inv{h}\langle I\rangle$, and $\forall B\subseteq V(I),
\inv{h}[\exactreachings{I}{B}]=\exactreachings{\inv{h}\langle I\rangle}{B}$.

The observer's I-state graph is a p-graph with events in the image space $X$. By
having $\Language{I} \supseteq h[\Language{W}]$, we are requiring that strings
generated in the world can be safely traced on $I$.

Next, we formalize the crucial connection from the
interaction of the robot and world, via the stream of symbols generated, to the
state tracked by the observer.  Inference proceeds from the observer back to the world,
though causality runs the other way (glance again
at~Figure~\ref{fig:modeloverview}). We begin, accordingly, with that latter
direction.

\begin{definition}[compatible world states]
\label{def:compatibleworldstates}
Given observer I-state graph $I$, robot's plan $(P, V_{\term})$, world graph
$(W, V_{\goal})$, and label map $h$, the world state $w$ is \emph{compatible}
with the set of I-states $B\subseteq V(I)$ if $\exists s\in \Language{W}$ such
that $s\in \underset{\strut(1)}{\underbrace{\inv{h}[\exactreachings{I}{B}]}}
\cap \underset{\strut(2)}{\underbrace{\Language{P}}} \cap
\underset{\strut(3)}{\underbrace{\reachings{W}{w}}}$.
\end{definition}
\vspace*{-16pt}

Informally, each of the three terms can be interpreted as:\vspace*{-3pt}
\begin{enumerate}[(1)]
\item An observer with I-state graph $I$ may ask which sequences are responsible
for having arrived at states $B$.  The answer is the set
$\exactreachings{I}{B}$, being the executions contained in equivalence classes
that are indistinguishable up to states in~$I$.  Those strings are in the image
space $X$, so, to obtain an answer in $Y\cup U$, we take their preimages. 
Every execution in $\inv{h}[\exactreachings{I}{B}]$ leads the
observer to $B$. Note that
information may be degraded by either $h$, $I$, or both. 
\gobble{
Figure~\ref{suppl:fig:infocollapse}
provides a visual example that shows how information can be degraded by a label
map $h$, an I-state graph $I$, and both together.  The first gives a scenario by
providing a world p-graph $W$, a plan $P$, and divulged plan information
$D$---all three are identical.  The second figure shows an I-state graph with
the same structure as $W$ and an identity label map. Every I-state corresponds
to a single world state in this case. In the third figure, there is an I-state
graph with the same structure as $W$, thus clearly possessing sufficient
structure to account for the world states. But here a label map conflates some
actions and some observations. A consequence is that the world states $w_1$ and
$w_2$ are indistinguishable given I-state $i_1$ and plan $P$. In the last figure
both $h$ and $I$ degrade information and do so independently. In this case,
$w_3$ and $w_4$ are indistinguishable owing to the label map, $w_5$ and $w_6$
are indistinguishable owing to the collapsed structure in $I$.}
\gobble{
\begin{figure}[t!]
\centering
\includegraphics[scale=0.78]{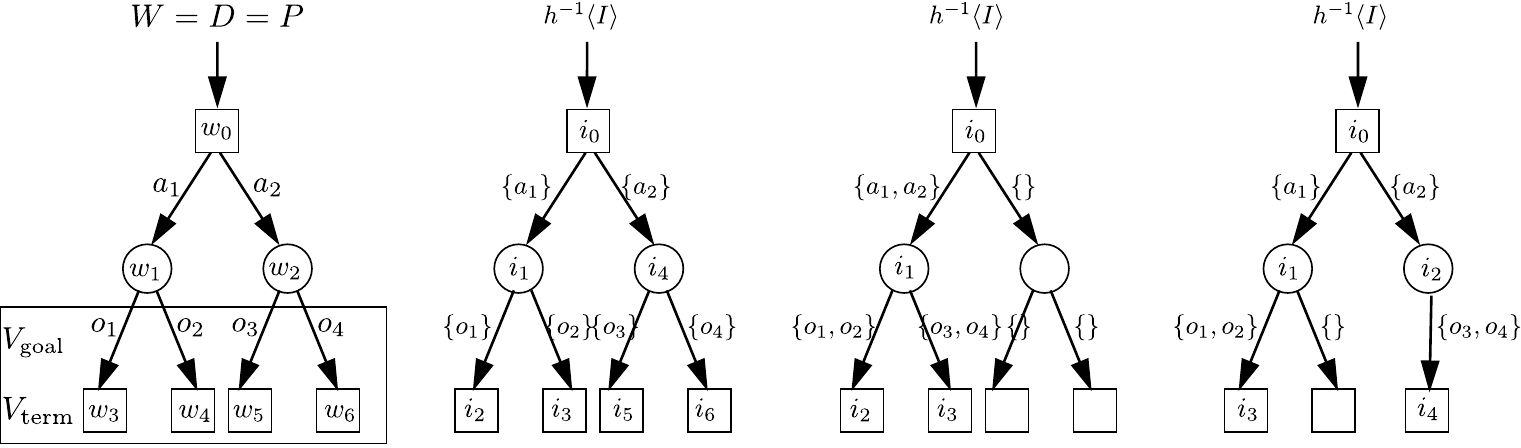}
\caption{Both the label map and the  I-state graph can degrade information.
Leftmost:~a scenario where the world p-graph $W$, a plan $P$, and divulged plan
information $D$ are all identical.  Second from the left: an I-state graph $I$
with the same structure as $W$ and an identity label map $h$.  Second from the
right: an I-state graph with $I$ the same structure as $W$ and a label map $h$
which conflates some actions/observations.  Rightmost:~both $h$
and $I$ degrade information independently.  \label{suppl:fig:infocollapse}
}
\vspace*{-10pt}
\end{figure}
}

\item The set of executions that may be executed by the robot is represented by
$\Language{P}$. If the observer knows that the robot's plan takes, say, the high
road, this information allows the observer to remove executions involving the
robot along the low road. 

\item The set of executions reaching world state $w$ is represented by
$\reachings{W}{w}$.  Two world states $w, w'\in V(W)$ are essentially
indiscernible or indistinguishable if ${\reachings{W}{w}=\reachings{W}{w'}}$, as
the sets capture the intrinsic uncertainty of world $W$.

\end{enumerate}

When an observer is in $B$, and $w$ is compatible with $B$, there
exists some execution, a certificate, that the world could plausibly be in $w$
subject to (1)~the current information summarized in $I$; (2)~the robot's plan;
(3)~the structure of the world. The set of \emph{all} world states that are
compatible with $B$ is denoted $\compatablew{I, P}{B}$, which is the observer's
estimate of the world states when known information about $W$, $P$ and $I$ have
all been incorporated.


A typical observer may know less about the robot's future behavior than the
robot's full plan. Weaker knowledge of how the robot will behave can be
expressed in terms of some p-graph $D$, such that ${\Language{D}\supseteq
\Language{P}}$.  (Here the mnemonic is that it is the divulged information about
the robot's plan, which one might imagine as leaked or communicated via a
side-channel.)
Notice that the information divulged to the observer about
the robot's execution is in the preimage space.
The key reason for this modeling
decision is that information may be lost under label map $h$; an observer 
gains the greatest information when the plan is disclosed in the preimage space and, as
we consider worst-case
conditions, we are interested in what the strongest  
(even adversarial) observers might infer. Thus, we study divulgence where the
observer obtains as much as possible.

Definition~\ref{def:compatibleworldstates} requires the substitution of
the second term in the intersection with $\Language{D}$.  When only $D$ is
given, the most precise inference replaces $\compatablew{I, P}{B}$ with
$\compatablew{I,D}{B}$:

\begin{definition}[estimated world states]
\label{def:estimatedworldstates}
Given an I-state graph $I$, divulged plan p-graph $D$, world p-graph $W$, and
label map $h$, the set of estimated world states for I-states $B\subseteq V(I)$
is $\compatablew{I,D}{B} \defeq \left\{w\in
V(W)\;\middle|\;(\exactreachings{\inv{h}\langle I\rangle}{B}\cap \Language{D}
\cap \reachings{W}{w})\neq \emptyset\right\}$.
\end{definition}
\vspace*{-6pt}

Observe that $\inv{h}[\exactreachings{I}{B}]$ has been replaced with
$\exactreachings{\inv{h}\langle I\rangle}{B}$, since
$\inv{h}[\exactreachings{I}{B}]=\exactreachings{\inv{h}\langle I\rangle}{B}$.
\label{sec:discussion_of_d}

The last remaining element in Figure~\ref{fig:dataflow} that needs to be
addressed is the stipulation of information. We do that next.\vspace*{-6pt}

\subsection{Information stipulations}

%
%


We prescribe properties of the information that an observer may extract from its
input by imposing constraints on the sets of estimated world states.  The
observer, filtering a stream of inputs sequentially, forms a correspondence
between its I-states and world states.  We write propositional formulas with
semantics defined in terms of this correspondence---in this model
the stipulations are written to hold over \emph{every} reachable set of
associated states.\footnote{We foresee other variants which are straightforward
to modifications to consider; but we report only on our current implementation.}
 
\begin{figure}[t]
{
\begin{equation*}
\begin{aligned}
&{\rm Formula} \rightarrow {\rm Clause}_1 \form{\land} \dots\form{\land} {\rm
Clause}_n\\[-2pt] &{\rm Clause}  \rightarrow {\rm Literal}_1 \form{\lor} \dots
\form{\lor} {\rm Literal}_m\\[-2pt]
&{\rm Literal} \rightarrow {\rm Symbol}\;|\; \form{\neg} {\rm Symbol}\\[-2pt]
&{\rm Symbol}  \rightarrow \form{\mathpzc{v_0}}, \form{\mathpzc{v_1}},
\form{\mathpzc{v_2}}, \dots\\[-8pt]
\end{aligned}
\end{equation*}}
\\
\setlength{\extrarowheight}{20pt}
\setlength{\tabcolsep}{12pt}
\centering
\begin{tabular}{cc}
\infer{ \langle \form{\mathpzc{v_i}} \rangle\Downarrow
\operatorname{eval}(v_i\stackrel{?}{\in} \compatablew{I,D}{B}
)}{[{\text{\footnotesize VALUE}}]~} & \infer{\langle \form{\neg \mathpzc{v_i}}
\rangle \Downarrow \text{the negation of }w}{[\textrm{\footnotesize NOT}]~~~~~~~
\langle \form{\mathpzc{v_i}}\rangle \Downarrow w~~~~~~~~~} \\ {\infer{ \langle
\form{\mathpzc{\ell_1} \lor \mathpzc{\ell_2}} \rangle \Downarrow \text{the
logical $\operatorname{or}$ of $w_1$ and $w_2$}}{[\textrm{\footnotesize OR}]
~~~~~~~~~ \langle \form{\mathpzc{\ell_1}} \rangle \Downarrow w_1~~~ \langle
\form{\mathpzc{\ell_2}}\rangle  \Downarrow w_2~~~~~~~~~}} & {\infer{\langle
\form{\mathpzc{c_1} \land \mathpzc{c_2}} \rangle \Downarrow \text{the logical
$\operatorname{and}$ of $w_1$ and $w_2$}}{[\textrm{\footnotesize AND}] ~~~~~~~~~
\langle \form{\mathpzc{c_1}} \rangle \Downarrow w_1~~~ \langle
\form{\mathpzc{c_2}}\rangle \Downarrow w_2~~~~~~}} \\ \end{tabular}
\caption{The syntax and natural semantics of the information stipulations,
where $\form{\mathpzc{c_i}}$, $\form{\mathpzc{\ell_i}}$, $\form{\mathpzc{v_i}}$,
represent a clause, literal, and symbol, respectively, and $w_i$ is the result
of the evaluation. The transition $\langle\form{\mathpzc{e}}\rangle \Downarrow
w$ denotes a transition, where $\form{\mathpzc{e}}$ is any expression defined by
the grammar and $w$ is the value yielded by the expression. }
\label{suppl:fig:operationalsemantics}
\vspace*{-1pt}
\end{figure}

First, however, we must delineate the scope of the estimated world states to be
constrained.  Some states, in inherently non-deterministic worlds, may be
inseparable because they are reached by the same execution.  In
Figure~\ref{fig:stringbehindvertices}, both $w_3$ and $w_4$ will be reached
(non-deterministically) by execution $a_2o_1$. Since this is intrinsic to the
world, even when the observer has perfect observations, they remain
indistinguishable. In the remainder of this paper, we will assume that the
world graph $W$ is in state-determined form, and we may affix stipulations to
the world states knowing that no two vertices will be non-deterministically
reached by the same execution.

Second, we write propositional formulae to constrain the observer's estimate.
Formula $\form{\Phi}$ is written in conjunctive normal form, consisting of
symbols, literals and clauses as shown in \fref{suppl:fig:operationalsemantics}.
Firstly, an atomic symbol $\form{\mathpzc{v_i}}$ is associated with each world
state $v_i\in V(W)$. If $v_i$ is contained in the observer's estimates
$\compatablew{I, D}{B}$, we will evaluate the corresponding symbol
$\form{\mathpzc{v_i}}$ as $\True$.  It evaluates as $\False$ otherwise. With
each symbol grounded in this way, we evaluate literals and clauses
compositionally, using logic operators {\sc not}, {\sc and}, {\sc or}. These are
defined in the standard way, eventually enabling evaluation of $\form{\Phi}$ on
the observer's estimate $\compatablew{I, D}{B}$.
\gobble{Suppose, for
example, we wish to require that state $v_1$ be included in the observer's
estimates of the world whenever $v_2$
is; this would be expressed via \form{$\Phi = \neg\mathpzc{v_2} \vee
\mathpzc{v_1}$}.  Evaluation of such formulas takes place as follows.  For a set
of I-states $B$ reached under some operation of the robot in the world,
\form{$\mathpzc{v_i}$} is connected with $v_i$ in that \form{$\mathpzc{v_i}$}
evaluates to \True for $B$ iff $v_i \in \compatablew{I,D}{B}$, where
$\compatablew{I,D}{B}$ is the set of estimated world states for I-states $B$.
The opposite condition, where $v_i\not \in \compatablew{I,D}{B}$, is written
naturally as \form{$\neg \mathpzc{v_i}$}.  Standard connectives \form{$\neg$},
\form{$\land$}, \form{$\lor$} enable composite expressions for complex
stipulations to be built.}

Let the predicate $\operatorname{satfd}(B,\form{\Phi})$ denote whether the
stipulation $\form{\Phi}$ holds for I-states $B$. Then a plan $P$ satisfies the
stipulations, if and only if
\vspace{-4pt}
\[
\forall s\in \Language{P}\cap \Language{W}\;
B=\reachedv{I}{h(s)}\,\;\operatorname{satfd}(B,\form{\Phi}). 
\vspace*{-5pt}
\]

\section{The observer's knowledge of the robot's plan}
\label{section:divulgedplan}

Above, we hinted that observers may differ depending on the prior knowledge that
has been revealed to them;  next we bring this idea into sharper focus.  The
information associated with an observer is contained in a pair $(I,D)$: the
I-state graph $I$ that acts as a filter, succinctly tracking state from a stream
of inputs, and knowledge of robot's plan in the form of a p-graph $D$.  These
two elements, through Definition~\ref{def:estimatedworldstates}, allow the
observer to form a correspondence with the external world $W$.  The I-state
graph $I$ induces $\veq{}{I}{}$ over its set of executions and hence over the
joint-executions with the world, or, more precisely, the image of those through
$h$. \gobble{(Recall that the coarseness of $h$ limits the fidelity of the observer,
see Figure~\ref{suppl:fig:infocollapse}.)} By comparing the fineness of the
relations induced by two I-state graphs, one obtains a sense of the relative
coarseness of the two I-state graphs.  As the present paper describes methods
motivated by applications to robotic privacy, we model the most capable
adversary, taking the \emph{finest observer}, that is, one whose equivalence
classes are as small as possible.

\vspace*{-2pt}
\begin{definition}[finest observer]
Given world graph $W$ and the divulged plan $D$, an I-state graph $\fine{I}$ is
a \emph{finest observer} if for any I-state graph $I$, we have $\forall s\in
\Language{W}$, $\compatablew{\fine{I},D}{h(s)}\subseteq \compatablew{I,
D}{h(s)}$.
\end{definition}
\vspace*{-2pt}

\begin{restatable}[]{lemma}{finestobserver}
\label{lemma:finestobserver}
$h\langle W\rangle$ is a finest observer.
\end{restatable}
By way of a proof sketch, note that the observer only ever sees the image of the world under the label map $h$,
i.e.  $h\langle W\rangle$. The p-graph $h\langle W\rangle$ serves as a natural
I-state graph for a finest observer as it allows the observer to have
sufficient internal structure to keep track of every world state.

The second element in the observer pair is $D$, information disclosed about the
plan, and presumed to be known \emph{a priori}, to the observer.  Depending on
how much the observer knows, there are multiple possibilities here, from most-
to least-informed:
\begin{enumerate}
\renewcommand{\theenumi}{\Roman{enumi}}%
\item The observer knows the exact plan $P$ to be
executed.\label{item:exactplan} \item The plan to be executed is among a finite
collection of plans $\{P_1,P_2, \dots, P_n\}$.\label{item:setplan}
\item The observer may only know that the robot is executing \emph{some} plan,
that is, the robot is goal directed and aims to achieve some state in
$V_{\goal}$. \label{item:someplan} \item The observer knows nothing about the
robot's execution other than that it is on $W$. \label{item:wanderingrobot}
\end{enumerate}%
\vspace*{-1pt}
It turns out that a p-graph exists whose language expresses knowledge for each
of those cases (we omit the details here). Furthermore,
Section~\ref{sec:discussion_of_d} details how the observer's knowledge of the
world state ($\compatablew{I,D}{B}$) from I-states $B$ depends on
$\exactreachings{\inv{h}\langle I\rangle}{B}\cap \Language{D}\cap \Language{W}$,
a set of executions that arrive at $B$ in the I-state graph $I$. 
Because the observer uses $D$ to refine $\exactreachings{\inv{h}\langle
I\rangle}{B}$, when $\Language{P} \subsetneq \Language{D}$ the gap between the
two sets of executions represents a form of uncertainty. The ordering of the
four cases, thus, can be stated precisely in terms of language inclusion.

Now using the $D$ as appropriate for each case, one may examine whether a given
plan and disclosure policy solves the planning problem (i.e., achieves desired
goals in the world) while meeting the stipulations on information communicated.
Hence, we see that describing disclosed information via a p-graph $D$ is in fact
rather expressive. This section has also illustrated the benefits of being able
to use both interaction language and graph presentation views of the same
structure.

\vspace*{-6pt}
\section{Searching for plans and disclosure policy: the {\sc Seek} problems}
\label{section:search}
In this section, we will show how to search for a plan (together with the label
map).
\par
\medskip
    \noindent
    \fbox{
        \begin{minipage}{0.96\textwidth}
               \noindent\begin{minipage}[t]{0.6\textwidth}
                        \textbf{Problem:} {\nameseekp$\big((W, V_{\goal}), \var{x},
                            (\fine{I}, D), h, \form{\Phi}\big)$\newline 
                        \hspace*{40pt}
                        \nameseekplm$\big((W, V_{\goal}), \var{x}, (\fine{I}, \var{x}),
                        \var{\lambda}, \form{\Phi}\big)$} \\[0.05in]
               \end{minipage}      
\hfill
               \noindent\begin{minipage}[t]{0.3\textwidth}
               \fbox{
                   \begin{minipage}[t]{0.75\textwidth}
                            {\footnotesize Vars. to solve for:\vspace*{-2pt}\\
                            \null~~\,$\var{x}$ is a plan\vspace*{-2pt}\\
                            \null~~~$\var{\lambda}$ is a label map}
                    \end{minipage}
               }
               \end{minipage}
           
               \noindent\begin{minipage}[t]{\textwidth}
                    \renewcommand{\tabcolsep}{2pt}
                    \begin{tabularx}{\linewidth}{rX}
                        \emph{Input:} & A planning problem $(W, V_{\goal})$, a
                        finest observer $\fine{I}$, a divulged plan p-graph $D$,
                        information disclosure policy $h$ and information stipulation
                        $\form{\Phi}.$\\
                        \emph{Output:} & A plan ${\var{x} = (P, V_{\term})}$ and/or
                        label map ${\var{\lambda}=h}$ such that plan $(P, V_{\term})$
                        solves the problem $(W, V_{\goal})$, and $\forall s\in
                        \Language{W^{\dagger}}\cap\Language{P}, B=\reachedv{I}{h(s)}$,
                        the information stipulation $\form{\Phi}$ is always evaluated as
                        \True on $\compatablew{I,D}{B}$ (i.e.
                        $\operatorname{satfd}(B,\form{\Phi})=\True$), else \False.
                    \end{tabularx}
                \end{minipage} 
                
        \end{minipage} 
    }


Of the two versions of {\sc Seek}, the first searches for a plan, the second
for a plan and a label map, jointly.  We consider each in turn.

\vspace*{-5pt}   
\subsection{Finding a plan given some predetermined $D$}
\label{section:seekplan}

For \nameseekp\xspace, first we must consider the search space of plans.  Prior
work~\cite{saberifar18pgraph} showed that, although planning problems can have
stranger solutions than people usually contemplate, there is a core of
well-structured plans (called homomorphic solutions) that suffice to determine
solvability. As an example, there may exist plans which loop around the
environment before achieving the goal, but, they showed that in seeking plans,
one need only consider plans that short-circuit the loops. 

The situation is rather different when a plan must satisfy more than mere goal
achievement: information stipulations may actually {\sl require} a plan to loop
in order to ensure that the disclosed stream of events is appropriate for the
observer's eyes. (A concrete example appears in \fref{fig:exp-nuclear}(c).)  
The argument in~\cite{saberifar18pgraph} needs modification for our problem---a
different construction can save the result even under disclosure constraints.
This fact is key to be able to implement a solution.


In this paper, without loss of generality, we focus on finding plans in
state-determined form.  Next, we will examine the solution space closely.

\begin{definition}
\label{def:congruent}
A plan $P$ is \emph{congruent} on the world graph $W$, if and only if for every
pair of executions $s_1, s_2\in \Language{P}$ we have $\veq{s_1}{P}{s_2}
\implies \veq{s_1}{W}{s_2}$.
\end{definition}
\vspace*{-5pt}

In other words, a plan that respects the equivalence classes of the world graph
is defined as a congruent plan.  Next, our search space is narrowed further
still.

\begin{lemma}
\label{lemma:congruentplan}
Given any plan $(P, V_{\term})$, there exists a plan $(P', V'_{\term})$ that is
congruent on the world graph $W$ and $\Language{P'}=\Language{P}$.
\end{lemma}
\begin{proof}
We give a construction from $P$ of $P'$ as a tree, and show that it meets the
conditions.
To construct $P'$, perform a BFS on $P$. Starting from $V_0(P)$, build a
starting vertex $v_0$ in $P'$, keep a correspondence between it and $V_0(P)$.
Mark $v_0$ as unexpanded. Now, for every unexpanded vertex $v$ in $P'$, mark the
set of all outgoing labels for its corresponding vertices in $P$ as $L_v$,
create a new vertex $v'$ in $P'$ for each label $l\in L_v$, build an edge from
$v$ to $v'$ with label $l$ in $P'$, and mark it as expanded. Repeat this process
until all vertices in $P'$ have been expanded. Mark the vertices corresponding
to vertices in $V_{\term}$ as $V'_{\term}$. In the new plan $(P', V'_{\term})$,
no two executions reach the same vertex. That is, $\forall s_1, s_2\in
\Language{P'}, \nveq{s_1}{P'}{s_2}$. Hence, $P'$ is congruent on $W$.  In
addition, since no new executions are introduced and no executions in $P$ are
eliminated during the construction of $P'$, we have
$\Language{P'}=\Language{P}$. \qed\end{proof}

\begin{restatable}[]{theorem}{boundcongruent}
\label{thm:boundcongruent}
For problem \nameseekp$\big((W, V_{\goal}), \var{x}, (I, D), h, \form{\Phi}\big)$, if
there exists a solution $(P, V_{\term})$, then there exists a solution $(P',
V'_{\term})$ that is both $c$-bounded and congruent on $W$, where $c=|V(W)|
\cdot |V(D)|\cdot |V(I)|$.
\end{restatable}
\vspace*{-10pt}
\begin{proof}
Suppose \nameseekp\xspace has a solution $(P, V_{\term})$.  Then the existence
of a solution $(P', V'_{\term})$ which is congruent on $W$ is implied by
Lemma~\ref{lemma:congruentplan}. Moreover, we have 
{\sc Check}$\big((W, V_{\goal}), (P, V_{\term}), D, I, h, \form{\Phi}\big)
\implies$
{\sc Check}$\big((W, V_{\goal}),$ $(P', V'_{\term}),$ $D,$ $ I,$ $ h,
\form{\Phi}\big)$,
following from two observations:
\begin{tightenumerate}
\item[(\textit{i}.)] if $(P, V_{\term})$ solves $(W, V_{\goal})$ then the means
of construction ensures $(P', V'_{\term})$ does as well, and

\item[(\textit{ii}.)] in checking $\form{\Phi}$, the set of estimated world
states $\compatablew{I,D}{\{v\}}$ does not change for each vertex $v\in
V(\sde{I})$, since the triple graph is independent of the plan to be searched.
The set of I-states to be evaluated by $\form{\Phi}$ in $\sde{I}$ is
$\cup_{s'\in h[\Language{P}\cap \Language{W}]}\reachedv{\sde{I}}{s'}$.  Since
$\Language{P}=\Language{P'}$, the set of I-states to be evaluated is no altered
and the truth of $\form{\Phi}$ along the plan is preserved.
\end{tightenumerate}
\vspace*{4pt}

\noindent The final step is to prove that if there exists a congruent solution
$(P', V'_{\term})$, then there exits a solution $(P'', V''_{term})$ that is
$c$-bounded. First, build a product graph $T$ of $W$, $D$, and $\inv{h}\langle
SED(I)\rangle$, with vertex set $V(W)\times V(D)\times V(\inv{h}\langle
\sde{I}\rangle)$. Then trace every execution $s$ in $P'$ on $T$. If $s$ visits
the same vertex $(v^W, v^{D}, v^{\inv{h}\langle \sde{I}\rangle})$ multiple
times, then $v^W$, $v^D$, and $v^{\inv{h}\langle \sde{I}\rangle}$ have to be
action vertices, for otherwise $P'$ can loop forever and is not a solution
(since $P'$ is finite on $W$).  Next, record the action taken at the last visit
of $(v^W, v^P, v^{\inv{h}\langle \sde{I}\rangle})$ as $a_{\rm last}$.  Finally,
build a new plan $(P'', V'_{\term})$ by bypassing unnecessary transitions on
$P'$ as follows. For each vertex $(v^W, v^P, v^{\inv{h}\langle \sde{I}\rangle})$
that is visited multiple times, $P''$ takes action $a_{\rm last}$ when $(v^W,
v^P, v^{\inv{h}\langle \sde{I}\rangle})$ is first visited.  $P''$ terminates at
the goal states without violating any stipulations, since it takes a shortcut
in the executions of $P'$ but---crucially---without visiting any new observer
I-states. In addition, $P''$ will visit each vertex in $T$ at most once, and the
maximum length of its executions is $|V(W)|\times |V(D)| \times
|V(\inv{h}\langle \sde{I}\rangle)|$.  Since $P''$ preserves the structure of
$P'$ during this construction, $P''$ is also congruent. 
\qed\end{proof}


The intuition, and the underlying reason for considering congruent plans, is
that modifying the plan will not affect the stipulations if the underlying
languages are preserved.  The bound on the length then takes this further,
modifying the language by truncating long executions in the triple graph,
thereby shortcutting visits to I-states that do not affect goal achievement.

Accordingly, it suffices to look for congruent plans in the (very specific) form
of trees, since any plan has a counterpart that is congruent and in the form of
a tree (see Lemma~\ref{lemma:congruentplan} for detail).
Theorem~\ref{thm:boundcongruent} states that the depth of the tree is at most
$c=|V(W)|\cdot |V(D)| \cdot |V(\inv{h}\langle \sde{I}\rangle)|$.
Therefore, we can limit the search space to trees of a specific bounded depth. 
To search for a $c$-bounded solution, first we mark the vertex $(v^W, v^{D},
v^{\inv{h}\langle \sde{I}\rangle})$ as: (i)~a goal state if $v^{W}$ is
a goal state in the world graph;
(ii)~as satisfying $\form{\Phi}$ when all the world states appearing together
with $v^{\inv{h}\langle \sde{I}\rangle}$ together satisfy $\form{\Phi}$.  Then
we will conduct an \aNd--\Or search~\cite{pearl84heuristics} on the triple
graph:
\vspace*{-4pt}
\begin{itemize}
\item [$\bullet$]Each action vertex serves as an \Or node, and an action should
be chosen for the action vertex such that it will eventually terminate at the
goal states and all the vertices satisfy $\form{\Phi}$ along the way.
\item [$\bullet$]Each observation vertex is treated as an \aNd node, and there
exists a plan that satisfies $\form{\Phi}$ for all its outgoing observation
vertices.
\end{itemize}

\vspace*{-8pt}
\subsection{Search for plan and label map for the finest observer, disclosing
the same} 
\label{section:seekplanlabelmap}

It is not merely the joint search that makes this, the second problem more
interesting. Whereas the first has a divulged plan $D$ that is \emph{a
priori} fixed, the second uses $\var{x}$, the plan that was found, as $D$. This
latter fact makes the third substantially more difficult.

At a high level, it is not hard to see why: the definitions in the previous
section show that both $P$ and $D$ play a role in determining whether a plan
satisfies a stipulation.  Where $D$ is known and fixed beforehand (for example,
in Case~\ref{item:wanderingrobot}, $D=W$, or Case~\ref{item:someplan},
$D=P^{*}$), a solution can proceed by building a correspondence in the triple
graph ${W\pgprod D\pgprod \inv{h}\langle \sde{I}\rangle}$ and searching in this
graph for a plan.  In \nameseekplm\xspace,  however, one is interested in the
case where ${D=P}$, where the divulged plan is tight, being the robot's plan
exactly.  We cannot search in the same product graph, because we can't make the
correspondence since $D$ has yet to be discovered, being determined only after
$P$ has been found. Crucially, the feasibility of $P$ depends on $D$, that is,
on itself! Finding such a solution requires an approach capable of building
incremental correspondences from partial plans.  A key result of this paper is
that \nameseekplm\xspace is actually solvable without resorting to mere
generate-and-check.


\begin{lemma}
\label{lemma:plan_state}
Let $\compatablew{}{}$ be estimated world states for the finest observer
$h\langle W\rangle$, and let~$w$ be the world state which is observable to the
robot.  If there exists a solution for \nameseekplm, then there exists a
solution that only visits each pair $(w, \compatablew{}{})$ at most once.
\end{lemma}
\begin{proof}
Let $(P, V_{\term})$ and $h$ be a solution for \nameseekplm. Suppose $P$ visited
$(w, \compatablew{}{})$ n times. Let the set of actions taken at $i$-th visit be
$A_i$. Then we can construct a new plan $(P', V_{\term})$ which always takes
$A_n$ at $(w, \compatablew{}{})$. If $P$ does not violate the stipulations, then
$P'$ will never do since $P'$ is a shortcut of $P$ and never visits more
I-states than $P$ does. In addition, $P'$ will also terminate at the goal region
if $P$ does.  \qed
\end{proof}

\begin{restatable}[]{theorem}{planstatescope}
If there exists a solution for \nameseekplm$\big((W, V_{\goal}), \var{x}, (I_f,
\var{x}), \var{\lambda}, \form{\Phi}\big)$. then there exists a plan $P$ that
takes $(w, \compatablew{h\langle W\rangle, P}{B})$ as its plan state, where $w$
is the world state and the set $\compatablew{h\langle W\rangle, P}{B}$ consists
of the estimated world states for I-states $B$. Furthermore, if $(w,
\compatablew{h\langle W\rangle, P}{B})\in V(P)$, then $\forall w'\in
\compatablew{h\langle W\rangle, P}{B}$, $(w', \compatablew{h\langle W\rangle,
P}{B})\in V(P)$. 
\end{restatable}
\vspace*{-8pt}
\begin{proof}
Lemma~\ref{lemma:plan_state} shows that we can treat $(w, \compatablew{h\langle
W\rangle, P}{B})$ as the plan state for the plan to be searched for.

Since $w'\in \compatablew{h\langle W\rangle, P}{B}$, we have $\exists s\in
\reachings{W}{w'}\cap \Language{P}\cap\inv{h}[\exactreachings{h\langle
W\rangle}{B}]$. Since $s\in \Language{P}$, $s$ reaches $w$ and $h(s)$ reaches
$B$, we have $s$ reaches the tuple $(w', \compatablew{h\langle W\rangle,
P}{B})$. Hence, $(w', \compatablew{h\langle W\rangle, P}{B})\in V(P)$.\qed
\end{proof}

%
\noindent In searching for $(P, V_{\term})$, for any action state ${v_p =
(w,\compatablew{h\langle W\rangle, P}{B})}$, we determine: 
\begin{description}
\item [\quad$w\in V_{\goal}:$] We must decide whether $v_p \in V_{\term}$ holds
or not;
\item [\quad$w\not\in V_{\goal}:$] We must choose the set of nonempty actions to
be taken at $v_p$. It has to be a set of actions, since these chosen actions are
not only aiming for the goal but also obfuscating each other under the label
map. 

 
\vspace*{-3pt}
\end{description}
\vspace*{-5pt}
A state ${v_p=(w, \compatablew{h\langle W\rangle, P}{B})}$ is a terminating
state in the plan when $\compatablew{h\langle W\rangle, P}{B}\subseteq
V_{\goal}$.

With action choices for each plan state $(w, \compatablew{h\langle W\rangle,
P}{B})$ and label map $h$, we are able to maintain transitions of the estimated
world states for $B'$ after observing the image~$x$. Now, if $(w,
\compatablew{h\langle W\rangle, P}{B})$ is an action state, let the set of
actions taken at $w$ be $A_w$. Then the label map $h$ partitions the actions in
$\cup_{w\in \compatablew{h\langle W\rangle, P}{B}} A_w$ into groups, each of
which shares the same image. The estimated worlds states for $B'$ transition in
terms of groups 
\vspace*{-6pt}
{\small
\begin{equation*}
\begin{split}
\compatablew{h\langle W\rangle, P}{B'}=\left\{w'\in V(W)\middle| (w,
\compatablew{h\langle W\rangle, P}{B}\right.)&\not\in V_{\term}, w\in
\compatablew{h\langle W\rangle, P}{B},\\[-7pt] &\left.\exists a\in A_w, h(a)=x,
\trto{W}{w}{a}{w'}\right\}.
\end{split}
\end{equation*}
\vspace*{-6pt}
}

Conversely, if $(w, \compatablew{h\langle W\rangle, P}{B})$ is an observation
state, let the observations available at $w$ be $O_w$. Then $h$ also partitions
the observations in $\cup_{(w, \compatablew{h\langle W\rangle, P}{B})\not\in
V_{\term}} O_w$ and estimated world states for $B'$ transition as
\vspace*{-2pt}
{\small
\begin{equation*}
\begin{split}
 \compatablew{h\langle W\rangle, P}{B'}=\left\{w'\in V(W)\middle| (w,
 \compatablew{h\langle W\rangle, P}{B}\right.)&\not\in V_{\term}, w\in
 \compatablew{h\langle W\rangle, P}{B},\\[-7pt] &\left.\exists o'\in O_w,
 h(o')=x, \trto{W}{w}{o}{w'}\right\}.
\end{split}
\end{equation*}
\vspace*{-14pt}
}

Instead of searching for the label map over the set of all actions and
observations in $W$, we will first seek a partial label map for all observations
or chosen actions for world states in $\compatablew{h\langle W\rangle, P}{B}$,
and then incrementally consolidate them. Each partial label map is a partition
of the events, making it easy to check whether two partial maps conflict when
they are consolidated. 
If two partial partitions disagree on a value, we backtrack in the search to try
another partition label map.  Putting it all together as detailed in
\fref{fig:searchtree}, we can build a type of \aNd--\Or search tree to
incorporate these choices.

\begin{SCfigure}
\centering
\includegraphics[scale=0.28]{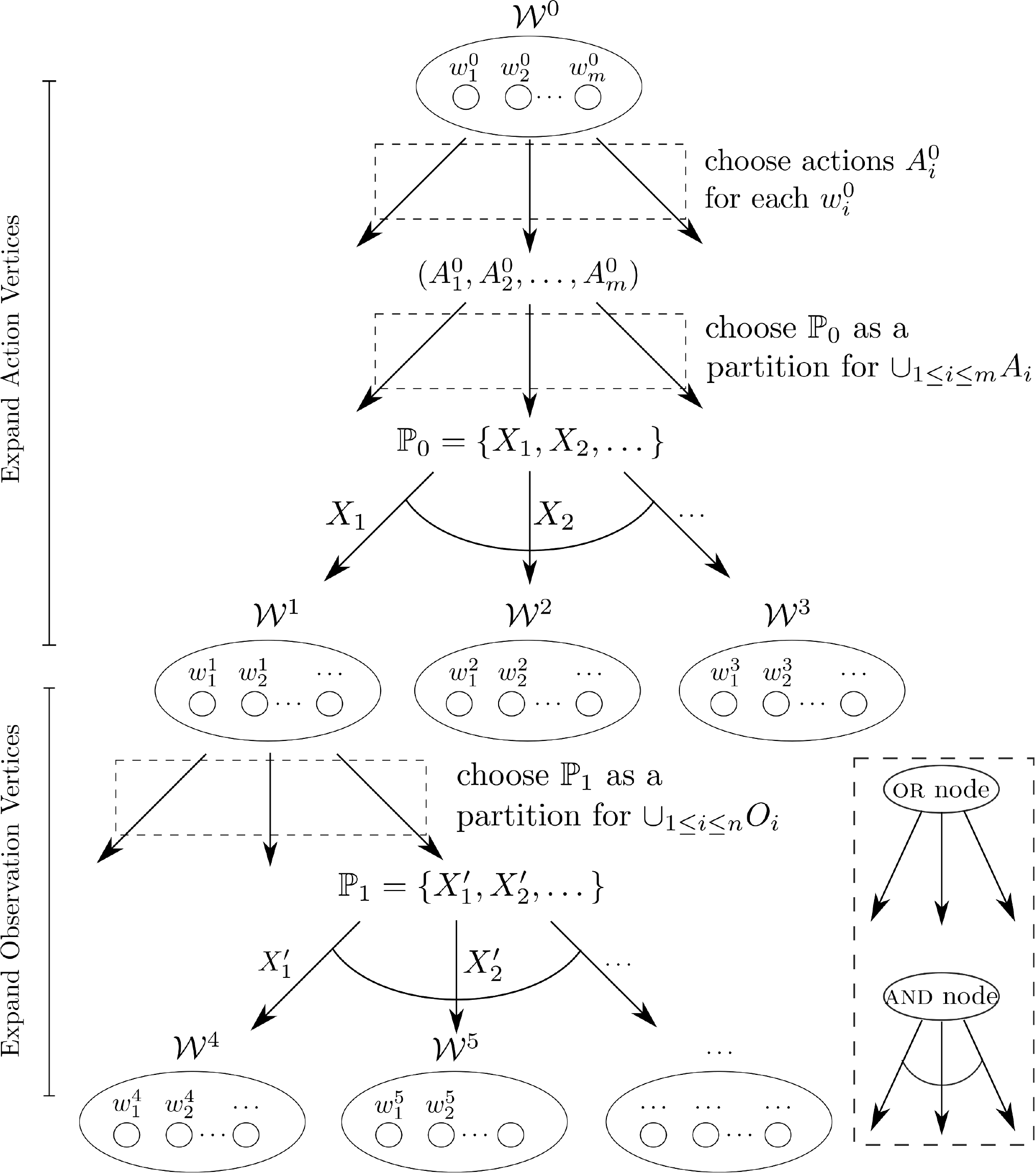}\vspace*{-10pt}
\caption{Solving the \nameseekplm\xspace problem via generalized \aNd--\Or search.%
\vspace{2pt}\newline\small
{
\null\quad For a set of actions comprising a vertex $\compatablew{0}{}$ two
tiers of \Or nodes are generated. The first is over subsets $(A^{0}_1,
A^{0}_2,\dots, A^{0}_m)$, being possible actions to the take; the second chooses
specific partitions of values $\mathbb{P}_i=\{X_1, X_2, \dots\}$, (i.e.,
partial label maps).  A given partition is expanded as an \aNd node with each
outgoing edge bearing a group of events sharing the same image under the partial
label map. \newline
\null\quad Observation vertices $\compatablew{1}{}$ are expanded in a similar
way, but are simpler since we forgo the step involving choosing actions.\newline
}
%
\label{fig:searchtree}
}
\end{SCfigure}



If there exists a plan and label map then for each $\compatablew{}{}$ in the
tree, there exists an action choice under which there exists a safe partition,
such that there exists a plan for all of its children. 
\gobble{The pesudocode to search
the plan by expanding the action node and observation node in the \aNd-\Or
search tree is shown in Algorithm~\ref{alg:actexpand} and \ref{alg:obsexpand}.}

Let the number of actions and observations in $W$ be $|Y|$ and $|U|$, and the
number of vertices be $|V|$.  There are $2^{|U||V|}$ action choices to consider,
in the worst case, for all the world states in $\compatablew{}{}$.
The total number of partitions is a Bell number $B_{|U|}$, where
$B_{n+1}=\sum^{n}_{k=0} C^{k}_n B_k$ and $B_0=1$. For each partition, the number
of groups we must consider is $|U|$. To expand an action vertex in the search
tree, the computation complexity is $2^{|U||V|}|U|B_{|U|}$. Similarly, the
complexity to expand an observation vertex is $|Y|B_{|Y|}$. If the depth of the
tree is $d$, then the computational complexity is $O(2^{{d|U||V|}})$.

\vspace*{-6pt}
\section{Experimental results}
\label{sec:experiments}
\vspace*{-4pt}

We implemented all the algorithms in this paper, the mainly using Python. The
problem \nameseekp\xspace was implemented with both the algorithm we propose and via
specification in computation tree logic (CTL) (and then utilizing the {\tt
nuXmv}  model-checker). All executions in this section used
a OSX laptop with a 2.4 GHz Intel Core i5 processor.

To experiment we constructed a ${3\times 4}$ grid for the nuclear inspection
scenario of \fref{fig:nuclear}.  Including the differing facility types and
radioactivity status, the world graph is a p-graph with $96$ vertices before
state-determined expansion ($154$ vertices for the state-determined form).  The
robot can move left, right, up, down one block at a time. After the robot's
movement, it receives $5$ possible observations: pebble bed facility or not
(only when located at the blue star), radioactivity high or low when located at
one of the `{\bf ?}' cells, and cell is an exit.  But the observer only knows
the image of the actions and observations under a label map. The stipulation
requires that the observer should learn the radioactivity strength, but should
never know the facility type.
\begin{figure}
\centering
\begin{subfigure}[c]{0.37\linewidth}
 \includegraphics[scale=0.4]{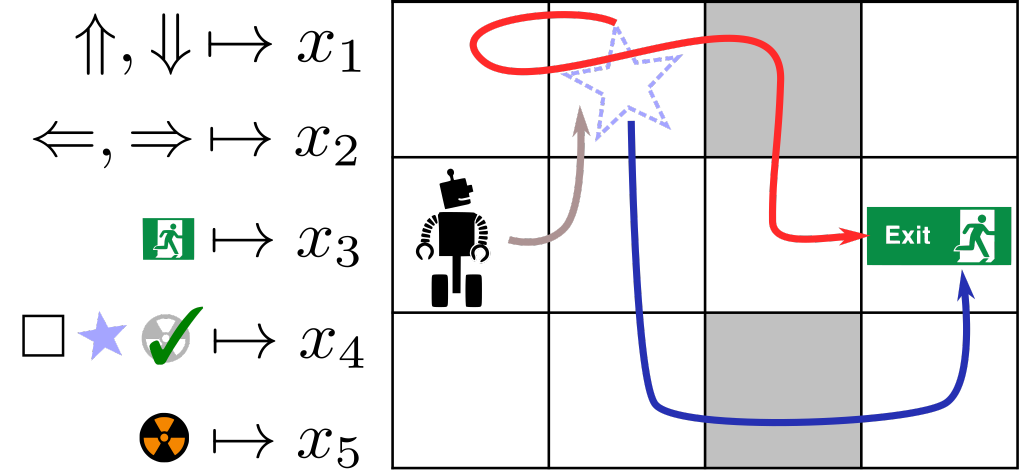}
 \caption{\label{fig:exp-nuclear-findplan}}
\end{subfigure}
\begin{subfigure}[c]{0.17\linewidth}
\includegraphics[scale=0.23]{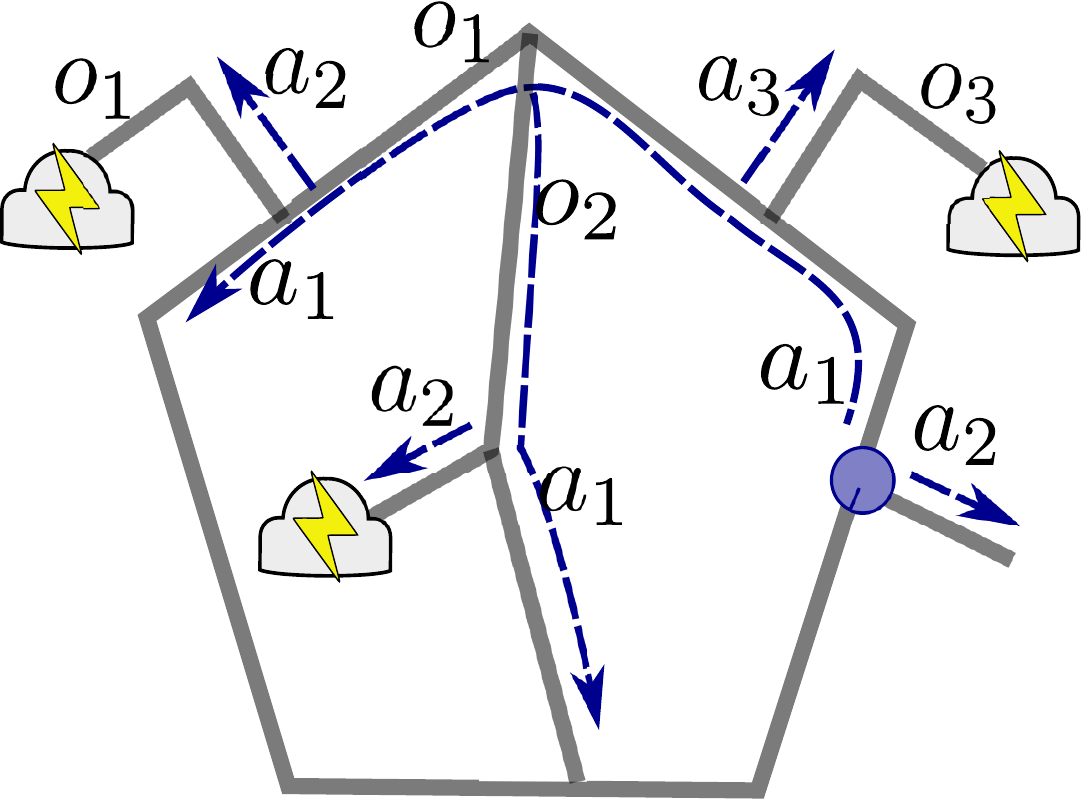}
\caption{\label{fig:exp-findplanlabelmap}}
\end{subfigure}
\caption{The scenario and results for \nameseekp\xspace and \nameseekplm\xspace problem:
(a) shows the plan found in the nuclear inspection scenario, when the observer
knows nothing about robot's plan (The robot traces the gray arrow, then the blue
one if blue light is seen, the red one otherwise.) (b) shows the pentagonal
world in \nameseekplm, where the robot moves along the gray lines.
\label{fig:exp-nuclear}}
\end{figure}
\vspace*{-20pt}
%

Firstly, we {\sc seek} the plan in the nuclear inspection scenario with a label
map shown in \fref{fig:exp-nuclear-findplan}. A plan can be found (with the
world graph disclosed, $D=W$). It takes $11$ seconds for the \aNd--\Or search
and $24$ seconds for the CTL-based implementation to find their solutions. The
CTL solver takes longer, but it prioritizes finding the plan of shortest length
first.
The plan found by CTL is shown in \fref{fig:exp-nuclear-findplan}. As the plan
found by \aNd--\Or search is lengthy, we omit it.


Since, for the nuclear inspection scenario, \nameseekplm\xspace doesn't return
any result within reasonable time we opted to examine a smaller problem.  Here a
robot moves in the pentagonal world shown in \fref{fig:exp-findplanlabelmap}.
The robot can either decide to loop in the world ($a_1$) or exit the loop at
some point ($a_2$ or $a_3$).  We wish to find a plan and label map pair so that
the robot can reach some charging station. The observer should not be able to
distinguish the robot's position when at either of the top two charging
locations.
\nameseekplm\xspace gives a plan which moves forward $6$ times and then exits at
the next time step. Additionally, to disguise the actions and observations after
the exit, it maps $h(a_2)=h(a_3)$ and $h(o_1)=h(o_3)$. Note that in this
problem, the robot reaches a goal, without considering the stipulations, by
taking the exit at the next time step. The stipulations force the robot to
navigate at least one loop in the world to conflate state for the sake of the
observer.

\vspace*{-8pt}
\section{Conclusion}
\label{sec:conclusion}
\vspace*{-6pt}

This paper continues a line of work on planning with constraints imposed on
knowledge- or belief-states. Our contribution is a substantial generalization of
prior models, though, as we see in the section reporting experiments, with grim
implications for computational requirements.  Future work might consider
techniques that incorporate costs, informed methods (with appropriate
heuristics), and other ways to solve certain instances quickly.

%

\vspace*{-6pt}
\bibliographystyle{IEEEtran} 
\bibliography{star_paper}

\end{document}